\documentclass{article} 
\usepackage{iclr2026_conference,times}

\usepackage{enumitem}
\usepackage{caption}
\usepackage{fancybox}

\usepackage{algpseudocode}

\usepackage[utf8]{inputenc} 
\usepackage[T1]{fontenc}    
\usepackage{hyperref}       
\usepackage{url}            
\usepackage{booktabs}       
\usepackage{amsfonts}       
\usepackage{nicefrac}       
\usepackage{microtype}      
\usepackage{xcolor}         
\definecolor{citecolor}{HTML}{0071BC}
\definecolor{linkcolor}{HTML}{D32F2F}
\definecolor{cellcolor}{HTML}{E3F2FD}
\definecolor{red}{HTML}{D32F2F}
\definecolor{magenta}{HTML}{D81B60}

\usepackage{amsmath}
\usepackage{amssymb}
\usepackage{mathtools}
\usepackage{amsthm}

\usepackage[capitalize,noabbrev]{cleveref}

\theoremstyle{plain}
\newtheorem{theorem}{Theorem}[section]
\newtheorem{proposition}[theorem]{Proposition}

\newtheorem{corollary}[theorem]{Corollary}
\theoremstyle{definition}

\newtheorem{assumption}[theorem]{Assumption}
\theoremstyle{remark}

\usepackage[textsize=tiny]{todonotes}
\usepackage{microtype}
\usepackage{graphicx}

\usepackage{subcaption}
\usepackage{booktabs} 

\usepackage{pgfplots}
\pgfplotsset{compat = newest}
\usepackage{multirow}

\usepackage{framed}
\usepackage{tcolorbox}

\usepackage{algorithm}
\usepackage{mathrsfs}
\usepackage{mathtools}
\usepackage{listings}
\usepackage{makecell}
\usepackage{colortbl}
\usepackage{color}
\usepackage{wrapfig}
\usepackage{cancel}
\usepackage{soul,xcolor}
\usepackage{pifont}
\usepackage{tikz}
\usetikzlibrary{shapes, positioning, arrows.meta, calc, decorations.pathmorphing, quotes}
\usepackage{todonotes}
\renewcommand{\cite}{\citep}

\usepackage{amsmath,amsfonts,bm}

\def\eqref#1{equation~\ref{#1}}

\def\1{\bm{1}}

\DeclareMathAlphabet{\mathsfit}{\encodingdefault}{\sfdefault}{m}{sl}
\SetMathAlphabet{\mathsfit}{bold}{\encodingdefault}{\sfdefault}{bx}{n}

\usepackage{hyperref}
\usepackage{url}

\title{A Theoretical Lens for RL-Tuned Language Models via Energy-Based Models}

\iclrfinalcopy

\author{
  Zhiquan Tan\textsuperscript{1}, Yinrong Hong\textsuperscript{2}\\
  1: Tsinghua University\\
  2: Beihang University\\
}

\begin{document}

\maketitle

\begin{abstract}

    Large language models (LLMs) trained via KL-regularized reinforcement learning demonstrate strong instruction following, self-correction, and reasoning abilities. Yet their theoretical underpinnings remain limited. We exploit the closed-form energy-based model (EBM) structure of the optimal KL-regularized policy to provide a unified variational analysis of LLMs.

    For instruction-tuned models, under natural assumptions on reward potentials and pretraining symmetry, we prove that the transition kernel satisfies detailed balance with respect to a scalar potential encoding response quality. This yields monotonic KL convergence to a high-quality stationary distribution, bounded hitting times to superior states, and exponential mixing governed by the spectral gap.

    For reasoning models trained with verifiable rewards (RLVR), we show the objective is equivalent to expected KL minimization toward an optimal reasoning distribution, with the suboptimality gap reducing to the Bernoulli KL between target and current accuracies along the natural gradient flow. This helps explain empirical entropy-accuracy trade-offs.

\end{abstract}

\section{Introduction}

The training of Large Language Models (LLMs) via Reinforcement Learning (RL) has fundamentally shifted the paradigm of generative AI \cite{openai2023gpt,touvron2023llama}. By moving beyond the simple next-token prediction objective of pretraining, alignment algorithms—such as Proximal Policy Optimization (PPO) \cite{schulman2017proximal}, Direct Preference Optimization (DPO) \cite{rafailov2023direct}, Group Relative Policy Optimization (GRPO) \cite{shao2024deepseekmath}: the model must maximize a reward signal while remaining close to a reference distribution. While the empirical success of this pipeline is undeniable, yielding models that are helpful, harmless, and increasingly capable of complex reasoning, the theoretical underpinnings of their behaviour remain under-explored. 

In this paper, we develop a principled theoretical framework that leverages the Energy-Based Model (EBM) \cite{LeCun2006} viewpoint to analyze both the {static} and {dynamic} behavior of RL trained LLMs, including both instruction-tuned and reasoning-capable LLMs. First, we theoretically analyze the transition structure induced by an instruction-tuned model induced by an instruction-tuned model $\pi_{inst}$ under plausible modeling assumptions. We show that the induced transition kernel satisfies a multiplicative detailed-balance relation with respect to an explicit potential $V$, which has been emperically observed by \cite{song2025detailed}. And later generalized to the case of self-correction \cite{kumar2024training}. The detailed-balance structure immediately yields monotonicity of the KL divergence to a canonical stationary distribution and gives quantitative bounds on hitting times to low-potential (high-quality) states. Spectral graph techniques then relate mixing rates to the graph Laplacian: the absolute spectral gap controls exponential convergence rates of expected potential, which in turn provides an interpretable explanation for observed differences in convergence speed across models.

Extending this framework to reasoning models trained via reinforcement learning with verifiable rewards (RLVR) \cite{wang2024math}, we derive an exact equivalence between the RLVR objective and minimization of the expected KL divergence to an optimal reasoning distribution. Under a univariate exponential family parameterization of the optimization trajectory—a natural consequence of natural gradient flow on the policy manifold—we prove that this KL gap reduces precisely to the Bernoulli KL divergence between target and current verification accuracies. This closed-form relationship yields theoretical insights into observed empirical phenomena, including the entropy-accuracy trade-off observed by \cite{cui2025entropy}.

\section{Related Works}

\paragraph{Energy-based models for LLMs.}
Energy-based formulations provide a natural, sequence-level view for modeling compatibility between full-token sequences and conditioning context, and they have recently been used both as alternative sequence generators and as analytical tools to explain LLM phenomena. Residual sequence-level EBMs (which add an unnormalized corrective energy on top of a pretrained locally-normalized generator) mitigate exposure bias and can improve generation quality and global coherence compared with purely autoregressive decoders \cite{Bakhtin2021}. More recent work has extended the EBM perspective to diffusion- and parallel-generation paradigms by introducing sequence-level energies at each diffusion step, recovering much of the sampling efficiency of diffusion approaches while closing the perplexity gap to autoregressive models \cite{Xu2024_edlm}. Beyond generation, scalable EBM training and sampling techniques (e.g., MCMC-scaling and implicit-generation schemes) illuminate why EBMs can capture long-range structure and compositional generalization—properties that are useful when one attempts to explain empirical LLM behaviours such as distributional shifts, spurious low-probability modes (hallucinations), and the effect of global reranking or calibration procedures \cite{Du2019,Song2021}.

\paragraph{Energy-based models in reinforcement learning.}
In reinforcement learning, the connection between soft (maximum-entropy) objectives and unnormalized Boltzmann policy forms naturally leads to energy-based policy representations. Early work shows that expressive, multimodal policies can be represented as EBMs over actions conditioned on states, and that soft Q-functions induce Boltzmann-form optimal policies which can be sampled or approximated via amortized samplers \cite{Haarnoja2017}. Subsequent lines of research have applied EBMs to both model-based planning (by scoring state-sequences or imagined trajectories) and to policy learning, using modern training techniques (SVGD, diffusion samplers, normalizing-flow hybrids) to make sampling and entropy estimation tractable in continuous, high-dimensional domains \cite{Du2019_planning,Chao2024_EBFlow}. Very recent algorithms fuse flow-based samplers or Stein variational updates with maximum-entropy objectives to obtain expressive yet trainable energy-based policies that improve sample efficiency and capture multimodality in action distributions; these approaches provide practical recipes and theoretical connections that are directly relevant for viewing LLM alignment / policy distillation problems through an EBM lens (e.g., seeing the aligned policy or critic as an energy landscape over sequences or behaviors) \cite{rafailov2023direct, Chao2024_EBFlow,Messaoud2024_S2AC}.

\section{Preliminary}

In KL-regularized reinforcement learning (RL), an important fact is that the optimal policy naturally takes the form of an energy-based model (EBM). This structure does not arise from a particular optimization algorithm—such as PPO, GRPO, or DPO—but is instead a direct consequence of the variational optimality of the objective itself. 

In a standard RL for LLM pipeline, training typically proceeds in two stages. In the first stage, a reward model $r(x,y)$ is trained using carefully created data, where $x$ denotes the input prompt and $y$ the model output. In the second stage, the reward model is fixed, and reinforcement learning is used to optimize a policy $\pi(y|x)$ that maximizes the reward while maintaining linguistic coherence and stability.

Naively maximizing the expected reward, $\mathbb{E}_{x \sim \mathcal{D},\, y \sim \pi(\cdot|x)}[r(x,y)]$, is generally insufficient and may even be harmful in practice. In modern RL formulations, this constraint is typically imposed via a Kullback--Leibler (KL) divergence penalty between the policy $\pi$ and a reference model $\pi_{\mathrm{ref}}$, which is usually a pretrained or (supervised) fine-tuned language model. The resulting optimization problem can be written as
\begin{equation}
\max_{\pi} \;
\mathbb{E}_{x \sim \mathcal{D},\, y \sim \pi(\cdot|x)} [r(x,y)]
\;-\;
\beta \, D_{\mathrm{KL}}\!\left(\pi(\cdot|x)\,\|\,\pi_{\mathrm{ref}}(\cdot|x)\right),
\label{eq:rlhf_obj}
\end{equation}
where $\beta>0$ controls the strength of the KL regularization and governs the trade-off between reward maximization and adherence to the reference model.

Intuitively, this objective encourages the policy to allocate more probability mass to high-reward outputs while penalizing deviations from the reference distribution. From a variational perspective, however, the KL-regularized form has a much deeper implication: it renders the optimal policy analytically tractable and endows it with a energy-based structure.

To see this, we rewrite the objective in Eq.~\eqref{eq:rlhf_obj}. Ignoring terms that are constant with respect to $\pi$, maximizing the objective is equivalent to minimizing
\[
\min_{\pi} \;
\mathbb{E}_{x \sim \mathcal{D}}
\left[
\mathbb{E}_{y \sim \pi(\cdot|x)}
\left[
\log \frac{\pi(y|x)}{\pi_{\mathrm{ref}}(y|x)}
\;-\;
\frac{1}{\beta} r(x,y)
\right]
\right].
\]
This expression reveals that the optimization problem is equivalent to minimizing the KL divergence between $\pi(y|x)$ and a particular reward-tilted distribution derived from $\pi_{\mathrm{ref}}$. 

Applying standard variational calculus with Lagrange multipliers (enforcing normalization independently for each $x$), the unique optimal solution is obtained in closed form \cite{rafailov2023direct}. The optimal policy $\pi^*(y|x)$ takes the form 
\begin{equation}
\pi^*(y | x)
=
\frac{1}{Z(x)} \,
\pi_{\mathrm{ref}}(y | x)
\exp\!\left( \frac{1}{\beta} r(x,y) \right),
\end{equation}
where
\[
Z(x)
=
\sum_y
\pi_{\mathrm{ref}}(y|x)
\exp\!\left( \frac{1}{\beta} r(x,y) \right)
\]
is the energy (partition) function that ensures normalization.

This result shows that, under KL-regularized RL, the optimal policy is exactly equivalent to a conditional energy-based model, with the reward function acting as (negative) energy and the reference model providing the base measure. This observation serves as the starting point for our subsequent analysis.

\section{Behavior of Instruction-Tuned Models}

In this section, we focus on the behavior of instruction-tuned models $\pi_{\mathrm{inst}}$, which are obtained after supervised fine-tuning (SFT) and reinforcement learning from human feedback (RLHF). Our goal is not to revisit the full training pipeline in engineering detail, but rather to establish a clean mathematical abstraction that allows us to analyze the induced dynamics of such models in a principled way.

For the purpose of theoretical analysis, we conceptually merge the SFT and RLHF stages into a single optimization procedure. This simplification is well justified: SFT can be viewed as a special case of reinforcement learning with a fixed dataset-induced reward, and both stages ultimately bias the pretrained model toward preferred generations under a reward signal. Concretely, we assume that the instruction-tuned model $\pi_{\mathrm{inst}}$ is obtained from a pretrained model $\pi_{\mathrm{pre}}$ by optimizing a reward function $r$, which assigns scalar scores to prompt–answer (or, more generally, state–state) pairs based on correctness, helpfulness, alignment with human preferences, and other qualitative criteria.

Throughout this section, we adopt an abstract state-space view. States $f,g \in \mathcal{S}$ may represent responses, latent semantic states, or coarse-grained functional configurations of the model. Transitions between states are induced by the stochastic policy $\pi_{\mathrm{inst}}$, and our interest lies in the structural properties of these transitions after reward-based optimization.

\begin{assumption}[Reward Structure]\label{ass:reward_sym}
For all state pairs $(f,g)$ under consideration, the reward admits a potential-based decomposition:
\[
r(f,g) = h(g) - h(f),
\]
where $h : \mathcal{S} \to \mathbb{R}$ is a scalar (potential) function.
\end{assumption}

This assumption encodes a particularly important class of rewards. Rather than assigning arbitrary scores to transitions, the reward measures relative improvement between states with respect to an underlying scalar objective $h$. Such potential-based rewards are standard in reinforcement learning theory and ensure that the reward does not introduce cyclic inconsistencies. In the context of instruction tuning, $h$ can be interpreted as a latent notion of answer quality, alignment, or correctness, and $r(f,g)$ quantifies how much better state $g$ is compared to $f$.

\begin{assumption}\label{ass:pretrain_sym}
For all state pairs $(f,g)$ under consideration,
\[
\log\frac{\pi_{\mathrm{pre}}(g \mid f)}{\pi_{\mathrm{pre}}(f \mid g)}
=
\log\frac{p_{\mathrm{data}}(g)}{p_{\mathrm{data}}(f)}.
\]
\end{assumption}

This assumption is natural when the pretrained model $\pi_{\mathrm{pre}}$ is a good approximation to the data-generating distribution. In particular, if $\pi_{\mathrm{pre}}(\cdot \mid \cdot) \approx p_{\mathrm{data}}(\cdot \mid \cdot)$, then
\[
\log\frac{\pi_{\mathrm{pre}}(g \mid f)}{\pi_{\mathrm{pre}}(f \mid g)}
\approx
\log\frac{p_{\mathrm{data}}(g \mid f)}{p_{\mathrm{data}}(f \mid g)}
=
\log\frac{p_{\mathrm{data}}(f,g)}{p_{\mathrm{data}}(g,f)}
+
\log\frac{p_{\mathrm{data}}(g)}{p_{\mathrm{data}}(f)}.
\]
Under appropriate coarse-graining (e.g., at the level of semantics, topics, or functional states), if the joint distribution is approximately symmetric,
$p_{\mathrm{data}}(f,g) \approx p_{\mathrm{data}}(g,f)$, or if the generation process is approximately reversible at this scale, then the first term vanishes. What remains is precisely the marginal log-density difference, which motivates Assumption~\ref{ass:pretrain_sym}.

Following notation in \cite{song2025detailed}, we define the transition kernel induced by the instruction-tuned model as
\[
T(g \mid f) = \pi_{\mathrm{inst}}(g \mid f).
\]
We are now ready to state the main structural result of this section.

\begin{theorem}\label{thm:db}
If Assumptions~\ref{ass:reward_sym} and~\ref{ass:pretrain_sym} hold, then there exists a potential function $V : \mathcal{S} \to \mathbb{R}$ such that, for all states under consideration,
\begin{equation}\label{eq:DB}
\log\frac{T(g \mid f)}{T(f \mid g)} = V(f) - V(g).
\end{equation}
\end{theorem}

\begin{proof}
Recall that under KL-regularized reinforcement learning, the optimal policy induced by the reward $r$ takes the form
\[
\pi_{\mathrm{inst}}(g \mid f)
=
\frac{1}{Z(f)} \, \pi_{\mathrm{pre}}(g \mid f)
\exp\!\left( \frac{r(f,g)}{\beta} \right),
\]
where
\[
Z(f) = \sum_{g'} \pi_{\mathrm{pre}}(g' \mid f)
\exp\!\left( \frac{r(f,g')}{\beta} \right)
\]
is the partition function ensuring normalization.

By direct algebraic manipulation, the log-ratio between the forward and backward transition probabilities can be written exactly as
\begin{equation}
\log \frac{T(g \mid f)}{T(f \mid g)}
=
\frac{1}{\beta} \big( r(f,g) - r(g,f) \big)
+
\log \frac{\pi_{\mathrm{pre}}(g \mid f)}{\pi_{\mathrm{pre}}(f \mid g)}
-
\big( \log Z(f) - \log Z(g) \big).
\label{eq:exact_ratio}
\end{equation}

Substituting Assumption~\ref{ass:reward_sym} and Assumption~\ref{ass:pretrain_sym} into~\eqref{eq:exact_ratio}, we obtain
\begin{align*}
\log \frac{T(g \mid f)}{T(f \mid g)}
&=
\frac{1}{\beta} \big( -h(f) + h(g) \big)
+
\big( \log p_{\mathrm{data}}(g) - \log p_{\mathrm{data}}(f) \big)
-
\big( \log Z(f) - \log Z(g) \big) \\
&=
\left(
\frac{-h(f)}{\beta}
-
\log p_{\mathrm{data}}(f)
-
\log Z(f)
\right)
-
\left(
\frac{-h(g)}{\beta}
-
\log p_{\mathrm{data}}(g)
-
\log Z(g)
\right).
\end{align*}

We now define the \textbf{potential function} $V : \mathcal{S} \to \mathbb{R}$ by
\[
V(s)
=
\frac{-h(s)}{\beta}
-
\log p_{\mathrm{data}}(s)
-
\log Z(s).
\]
Substituting this definition back into the expression above yields
\[
\log \frac{T(g \mid f)}{T(f \mid g)} = V(f) - V(g),
\]
which proves the claim. Equivalently, this relation can be written in the multiplicative detailed balance form
\[
\frac{T(g \mid f)}{T(f \mid g)} = \exp\!\big( V(f) - V(g) \big).
\]
\end{proof}

\begin{corollary}
When $h = 0$, the result applies to the special case studied in \cite{song2025detailed}.
\end{corollary}

\subsection{The Case of Self-Correction}

In general self-correction tasks \cite{kumar2024training}, the function $h$ is non-zero and typically biased toward improved states, meaning that $h(g) > h(f)$ when $g$ represents a refined or corrected version of $f$. Through its appearance in the definition of the potential $V$, this improvement-oriented structure induces a negative shift in $V$, and lower potential values are generally associated with better-performing states. By tracking how the potential evolves along chains of self-correction, one can further derive monotonicity and convergence properties, which will be formalized in the following theorem.

Let $\pi(s) \propto \exp(-V(s))$ be the stationary distribution, and recall the transition kernel $T(g \mid f)$ satisfy the detailed balance condition:
\begin{equation}
    \pi(f) T(g \mid f) = \pi(g) T(f \mid g), \quad \forall f, g \in \mathcal{S}.
    \label{eq:detailed_balance}
\end{equation}

We now prove that the distribution of the system states, $P_t$, strictly approaches the target distribution $\pi$ over time in terms of the Kullback-Leibler (KL) divergence. 

\begin{theorem}[Monotonic Decrease of KL Divergence]
Let $P_t$ denote the probability distribution of the state at time $t$, evolving according to the master equation $P_{t+1}(g) = \sum_f P_t(f) T(g \mid f)$. The KL divergence with respect to the target $\pi$ is non-increasing:
\begin{equation}
    \mathrm{KL}(P_{t+1} \| \pi) \le \mathrm{KL}(P_t \| \pi).
\end{equation}
\end{theorem}

\begin{proof}
Consider the convex function $\phi(x) = x \log x$. The KL divergence can be written as:
\begin{equation*}
    D_{\mathrm{KL}}(P_{t+1} \| \pi) = \sum_{g \in \mathcal{S}} \pi(g) \, \phi\left( \frac{P_{t+1}(g)}{\pi(g)} \right).
\end{equation*}
We expand the ratio $P_{t+1}(g)/\pi(g)$ using the following equation:
\begin{equation*}
    \frac{P_{t+1}(g)}{\pi(g)} = \frac{1}{\pi(g)} \sum_{f} P_t(f) T(g \mid f) = \sum_{f} \frac{P_t(f)}{\pi(f)} \frac{\pi(f) T(g \mid f)}{\pi(g)}.
\end{equation*}
Using the detailed balance condition $\frac{\pi(f) T(g \mid f)}{\pi(g)} = T(f \mid g)$, we obtain:
\begin{equation*}
    \frac{P_{t+1}(g)}{\pi(g)} = \sum_{f} T(f \mid g) \left( \frac{P_t(f)}{\pi(f)} \right).
\end{equation*}
This expresses the ratio at $t+1$ as a convex combination of ratios at $t$ (since $\sum_f T(f \mid g) = 1$). Applying Jensen's inequality for the convex function $\phi$:
\begin{align*}
    \mathrm{KL}(P_{t+1} \| \pi) &= \sum_{g} \pi(g) \, \phi\left( \sum_{f} T(f \mid g) \frac{P_t(f)}{\pi(f)} \right) \\
    &\le \sum_{g} \pi(g) \sum_{f} T(f \mid g) \, \phi\left( \frac{P_t(f)}{\pi(f)} \right).
\end{align*}
We now swap the summation order and apply detailed balance ($\pi(g) T(f \mid g) = \pi(f) T(g \mid f)$) again:
\begin{align*}
    \text{RHS} &= \sum_{f,g} \pi(f) T(g \mid f) \, \phi\left( \frac{P_t(f)}{\pi(f)} \right) \\
    &= \sum_{f} \pi(f) \, \phi\left( \frac{P_t(f)}{\pi(f)} \right) \underbrace{\sum_{g} T(g \mid f)}_{=1} \\
    &= \sum_{f} \pi(f) \left( \frac{P_t(f)}{\pi(f)} \log \frac{P_t(f)}{\pi(f)} \right) \\
    &= \sum_{f} P_t(f) \log \frac{P_t(f)}{\pi(f)} \\
    &= \mathrm{KL}(P_t \| \pi).
\end{align*}
Thus, $\mathrm{KL}(P_{t+1} \| \pi) \le \mathrm{KL}(P_t \| \pi)$.
\end{proof}

From the above theorem the distribution are moving towards the stationary distribution $\pi$. When it arrives at the stationary distribution $\pi$, while the single-step drift at a specific state may be non-zero (driving the system towards lower potential), the expected drift over the stationary distribution must vanish. So we may usually observe that the model stays in lower potential regime (high $\pi$ probability) and usually the lower potential reflects a better performance.

\begin{proposition}[Zero Global Mean Drift]
Let the single-step potential drift at state $f$ be defined as $\Delta(f) = \mathbb{E}[V(X_{t+1}) - V(X_t) \mid X_t = f]$. Under the stationary distribution $\pi$, the global expected drift is zero:
\begin{equation}
    \mathbb{E}_{f \sim \pi} \big[ \Delta(f) \big] = 0.
\end{equation}
\end{proposition}

\begin{proof}
We expand the total expectation of the drift:
\begin{align*}
    \mathbb{E}_{f \sim \pi} [\Delta(f)] &= \sum_{f \in \mathcal{S}} \pi(f) \sum_{g \in \mathcal{S}} T(g \mid f) \big( V(g) - V(f) \big) \\
    &= \underbrace{\sum_{f,g} \pi(f) T(g \mid f) V(g)}_{\text{Term A}} - \underbrace{\sum_{f,g} \pi(f) T(g \mid f) V(f)}_{\text{Term B}}.
\end{align*}
We analyze the two terms separately:

\textbf{Term A:} We apply the detailed balance condition \eqref{eq:detailed_balance}:
\begin{align*}
    \text{Term A} &= \sum_{f,g} \pi(g) T(f \mid g) V(g) \\
    &= \sum_{g} \pi(g) V(g) \underbrace{\sum_{f} T(f \mid g)}_{=1} \\
    &= \mathbb{E}_{\pi}[V].
\end{align*}
\textbf{Term B:} We use the normalization of the transition probability $\sum_g T(g \mid f) = 1$:
\begin{align*}
    \text{Term B} &= \sum_{f} \pi(f) V(f) \underbrace{\sum_{g} T(g \mid f)}_{=1} \\
    &= \mathbb{E}_{\pi}[V].
\end{align*}
Substituting these back into the original equation:
\begin{equation*}
    \mathbb{E}_{f \sim \pi} [\Delta(f)] = \mathbb{E}_{\pi}[V] - \mathbb{E}_{\pi}[V] = 0.
\end{equation*}
\end{proof}

The previous discussions mainly focus on qualitative arguments, so we will move on to some quantitative arguments. We can further derive the bound on the time required to reach a target set defined by a low potential value.
        
        \begin{theorem}[Bound on Hitting Time]
        \label{thm:hitting_time}
        Let $m = min_{s \in \mathcal{S}} V(s)$ be the global minimum of the potential. For a threshold $b \in \mathbb{R}$, define the target set $B = \{x \in \mathcal{S} : V(x) \le b\}$. Let $\tau_B = inf\{t \ge 0 : X_t \in B\}$ be the first hitting time of the set $B$.
        
        Assume there exists a constant $\gamma > 0$ such that for all states $f \notin B$:
        \begin{equation*}
            \Delta(f) \le -\gamma.
        \end{equation*}
        Then, for any initial state $X_0 = f$, the expected hitting time is bounded by:
        \begin{equation}
            \mathbb{E}_f[\tau_B] \le \frac{V(f) - m}{\gamma}.
        \end{equation}
        \end{theorem}
        
        \begin{proof}
        Let $\mathcal{F}_t$ denote the filtration generated by the history $\{X_0, \dots, X_t\}$. Since the hitting time $\tau_B$ is a random variable and potentially infinite, we introduce the stopped process $Y_t$ defined by:
        \begin{equation*}
            Y_t = V(X_{t \wedge \tau_B}),
        \end{equation*}
        where $t \wedge \tau_B = \min(t, \tau_B)$. Since the state space $\mathcal{S}$ is finite, the function $V$ is bounded, ensuring that the sequence $Y_t$ is well-defined and bounded.
        
        We analyze the conditional expectation of the increment $\mathbb{E}[Y_{t+1} - Y_t \mid \mathcal{F}_t]$. We distinguish two cases based on the relationship between the time $t$ and the stopping time $\tau_B$:
        
        \textbf{Case 1:} If $t \ge \tau_B$, then the process has already stopped. We have $X_{t \wedge \tau_B} = X_{\tau_B}$ and $X_{(t+1) \wedge \tau_B} = X_{\tau_B}$. Thus, the increment is zero:
        \begin{equation*}
            Y_{t+1} - Y_t = 0.
        \end{equation*}
        
        \textbf{Case 2:} If $t < \tau_B$, then $X_t \notin B$. By the definition of the stopped process, $Y_t = V(X_t)$. The next state is $X_{t+1}$ (which may or may not be in $B$). Applying the {drift condition}, we have:
        \begin{equation*}
            \mathbb{E}[Y_{t+1} - Y_t \mid \mathcal{F}_t, t < \tau_B] = \Delta(X_t) \le -\gamma.
        \end{equation*}
        
        We can combine these two cases using the indicator function $\mathbf{1}_{\{t < \tau_B\}}$, which is measurable with respect to $\mathcal{F}_t$:
        \begin{equation*}
            \mathbb{E}[Y_{t+1} - Y_t \mid \mathcal{F}_t] \le -\gamma \cdot \mathbf{1}_{\{t < \tau_B\}}.
        \end{equation*}
        Taking the total expectation on both sides of the inequality yields:
        \begin{equation*}
            \mathbb{E}[Y_{t+1}] - \mathbb{E}[Y_t] \le -\gamma \mathbb{P}(t < \tau_B).
        \end{equation*}
        We sum this inequality over the time steps $t = 0, \dots, n-1$:
        \begin{equation*}
            \sum_{t=0}^{n-1} \left( \mathbb{E}[Y_{t+1}] - \mathbb{E}[Y_t] \right) \le -\gamma \sum_{t=0}^{n-1} \mathbb{P}(t < \tau_B).
        \end{equation*}
        The left-hand side is a telescoping sum which simplifies to $\mathbb{E}[Y_n] - \mathbb{E}[Y_0]$. Since $X_0 = f$ and $0 < \tau_B$ (assuming $f \notin B$), $\mathbb{E}[Y_0] = V(f)$.
        The right-hand side can be rewritten using the identity for the expectation of non-negative integer-valued random variables:
        \begin{equation*}
            \sum_{t=0}^{n-1} \mathbb{P}(t < \tau_B) = \sum_{t=0}^{n-1} \mathbb{E}[\mathbf{1}_{\{t < \tau_B\}}] = \mathbb{E}\left[ \sum_{t=0}^{n-1} \mathbf{1}_{\{t < \tau_B\}} \right] = \mathbb{E}[\min(n, \tau_B)].
        \end{equation*}
        Substituting these back into the inequality, we obtain:
        \begin{equation*}
            \mathbb{E}[Y_n] - V(f) \le -\gamma \mathbb{E}[n \wedge \tau_B].
        \end{equation*}
        Rearranging the terms yields:
        \begin{equation*}
            \gamma \mathbb{E}[n \wedge \tau_B] \le V(f) - \mathbb{E}[Y_n].
        \end{equation*}
        By the definition of the global minimum $m$, we have $Y_n = V(X_{n \wedge \tau_B}) \ge m$. Therefore, $\mathbb{E}[Y_n] \ge m$. Substituting this lower bound into the inequality gives:
        \begin{equation*}
            \gamma \mathbb{E}[n \wedge \tau_B] \le V(f) - m.
        \end{equation*}
        Finally, we take the limit as $n \to \infty$. The sequence of random variables $Z_n = n \wedge \tau_B$ is non-negative and monotonically increasing towards $\tau_B$. By the Monotone Convergence Theorem, we have:
        \begin{equation*}
            \lim_{n \to \infty} \mathbb{E}[n \wedge \tau_B] = \mathbb{E}[\tau_B].
        \end{equation*}
        Consequently, we obtain the final bound:
        \begin{equation*}
            \gamma \mathbb{E}[\tau_B] \le V(f) - m \implies \mathbb{E}[\tau_B] \le \frac{V(f) - m}{\gamma}.
        \end{equation*}
        \end{proof}

      Then we will utilize Spectral Graph Theory to rigorously bound the convergence rate. We assume the state space graph $G=(\mathcal{S}, \mathcal{E})$ is connected and finite.

We define the {Expected Potential} at time $t$ as $L(t) = \mathbb{E}_{s \sim P_t}[V(s)]$, and the limit value as $L_\infty = \lim_{t \to \infty} L(t) = \mathbb{E}_{s \sim \pi}[V(s)]$.
        
\begin{theorem}[Absolute Spectral Gap Controls Convergence Rate]
    \label{thm:spectral_gap}
    Let \(\lambda_2\) be the second smallest eigenvalue of the normalized Laplacian \(\mathcal{L} = I - \mathbf{P}\), and let \(\rho = \max_{i \ge 2} |\mu_i|\), where \(\mu_i = 1 - \lambda_i\) are the eigenvalues of the self-adjoint transition operator \(\mathbf{P}\). For any initial distribution \(P_0\), the deviation of the expected potential decays exponentially:
    \begin{equation}
        |L(t) - L_\infty| \le \rho^t \cdot \sqrt{\mathrm{Var}_\pi(V)} \cdot \chi(P_0 \| \pi),
    \end{equation}
    where \(\chi(P_0 \| \pi) = \| P_0/\pi - 1 \|_{L^2(\pi)}\) is the chi-square distance between the initial and stationary distributions.
    \end{theorem}
    
    \begin{proof}
    Let \(\mathbf{P}\) be the transition operator acting on \(L^2(\pi)\). Under the detailed balance condition, \(\mathbf{P}\) is self-adjoint with eigenvalues \(1 = \mu_1 > \mu_2 \ge \dots \ge \mu_n \ge -1\), where the eigenfunction \(u_1\) associated with \(\mu_1\) is the constant \(\mathbf{1}\).
    
    Define the density ratio \(g_t = P_t / \pi\). The dynamics follow \(g_t = \mathbf{P}^t g_0\). The deviation from stationarity, \(g_t - 1\), is orthogonal to \(u_1\). Expanding in the eigenbasis, the norm on the orthogonal complement decays as:
    \begin{equation*}
        \| g_t - 1 \|_{L^2(\pi)} = \| \mathbf{P}^t (g_0 - 1) \|_{L^2(\pi)} \le \left( \max_{i \ge 2} |\mu_i| \right)^t \| g_0 - 1 \|_{L^2(\pi)} = \rho^t \| g_0 - 1 \|_{L^2(\pi)}.
    \end{equation*}
    Note that \(\| g_0 - 1 \|_{L^2(\pi)}\) is precisely \(\chi(P_0 \| \pi)\).
    
    We rewrite the potential expectation gap as an inner product in \(L^2(\pi)\):
    \begin{equation*}
        L(t) - L_\infty = \sum_{s} (P_t(s) - \pi(s)) V(s) = \sum_{s} \pi(s) (g_t(s) - 1) V(s) = \langle g_t - 1, V \rangle_\pi.
    \end{equation*}
    Since \(\langle g_t - 1, c \rangle_\pi = 0\) for any constant \(c\), we can replace \(V\) with the centered potential \(V - L_\infty\):
    \begin{equation*}
        L(t) - L_\infty = \langle g_t - 1, V - L_\infty \rangle_\pi.
    \end{equation*}
    
    Applying the Cauchy-Schwarz inequality:
    \begin{equation*}
        |L(t) - L_\infty| \le \| g_t - 1 \|_{L^2(\pi)} \cdot \| V - L_\infty \|_{L^2(\pi)}.
    \end{equation*}
    Substituting the bound and noting that \(\| V - L_\infty \|_{L^2(\pi)} = \sqrt{\mathrm{Var}_\pi(V)}\), we obtain:
    \begin{equation*}
        |L(t) - L_\infty| \le \rho^t \cdot \chi(P_0 \| \pi) \cdot \sqrt{\mathrm{Var}_\pi(V)}.
    \end{equation*}
    \end{proof}
        
While $\lambda_2$ is hard to compute, we can bound it from the following inequality.       
        \begin{theorem}[Poincaré Inequality]
        The variance of the potential under the stationary distribution is bounded by the expected squared gradient (Dirichlet form):
        \begin{equation}
            \mathrm{Var}_\pi(V) \le \frac{1}{\lambda_2} \mathbb{E}_{s \sim \pi} \left[ \frac{1}{2} \sum_{s'} T(s' \mid s) (V(s) - V(s'))^2 \right].
        \end{equation}
        \end{theorem}
        
        \begin{proof}
        The spectral gap $\lambda_2$ is characterized by the variational principle (Rayleigh quotient):
        \begin{equation*}
            \lambda_2 = \inf_{f \perp \mathbf{1}, f \neq 0} \frac{\langle f, \mathcal{L} f \rangle_\pi}{\langle f, f \rangle_\pi}.
        \end{equation*}
        Let $f = V - L_\infty$. Then $f \perp \mathbf{1}$ and $\langle f, f \rangle_\pi = \mathrm{Var}_\pi(V)$. The numerator is the Dirichlet form $\mathcal{E}(V,V)$. The inequality follows immediately from the definition of the infimum.
        \end{proof}

From the above inequality. we can find that the $\lambda_2$ for Gemini and Claude may be bigger than that of ChatGPT, so explains the convergence speed observed in \cite{song2025detailed}.

\section{Behavior of Reasoning Models}

Before discussing dynamics (time evolution), we first establish a fundamental static result:
{maximizing the RLVR objective is strictly equivalent to minimizing the KL divergence between the policy $\pi$ and the optimal distribution $\pi_{\mathrm{reas}}$.}

Let $\mathcal{D}$ denote the data distribution over prompts $x$.  
The RLVR objective can be seen as
\begin{equation}
\label{eq:global_objective}
\mathcal{J}(\pi)
\;=\;
\mathbb{E}_{x \sim \mathcal{D}}
\Big[
\mathbb{E}_{y \sim \pi(\cdot \mid x)} \big[ r(x,y) \big]
\;-\;
\beta \, \mathrm{KL}\!\left(
\pi(\cdot \mid x)\, \| \, \pi_{\mathrm{inst}}(\cdot \mid x)
\right)
\Big],
\end{equation}
where $\beta>0$ is the KL regularization coefficient.

We now derive a geometric representation of the reward via a dual (energy-based) identity.
Taking logarithms on both sides of the optimal policy yields
\begin{equation*}
\log \pi_{\mathrm{reas}}(y \mid x)
=
\log \pi_{\mathrm{inst}}(y \mid x)
+
\frac{1}{\beta} r(x,y)
-
\log Z(x).
\end{equation*}

Rearranging terms gives an explicit expression for the reward:
\begin{equation}
\label{eq:reward_identity}
r(x,y)
=
\beta \log \frac{\pi_{\mathrm{reas}}(y \mid x)}{\pi_{\mathrm{inst}}(y \mid x)}
+
\beta \log Z(x).
\end{equation}

Substituting \eqref{eq:reward_identity} into the reward term of
$\mathcal{J}(\pi)$ in \eqref{eq:global_objective}, we obtain
\begin{align*}
\mathbb{E}_{y \sim \pi} [ r(x,y) ]
&=
\beta \,
\mathbb{E}_{y \sim \pi}
\left[
\log \frac{\pi_{\mathrm{reas}}(y \mid x)}{\pi_{\mathrm{inst}}(y \mid x)}
\right]
+
\beta \log Z(x).
\end{align*}

Combining this with the KL regularization term, and noting that
\begin{equation*}
\mathrm{KL}(\pi \| \pi_{\mathrm{inst}})
=
\mathbb{E}_{y \sim \pi}
\left[
\log \frac{\pi(y \mid x)}{\pi_{\mathrm{inst}}(y \mid x)}
\right],
\end{equation*}
the $\log \pi_{\mathrm{inst}}$ contributions cancel exactly. After simplification, we arrive at
\begin{equation}
\label{eq:energy_identity}
\mathcal{J}(\pi)
=
\mathcal{J}(\pi_{\mathrm{reas}})
-
\beta \,
\mathbb{E}_{x \sim \mathcal{D}}
\left[
\mathrm{KL}\!\left(
\pi(\cdot \mid x)\, \| \, \pi_{\mathrm{reas}}(\cdot \mid x)
\right)
\right].
\end{equation}

Then Equation \eqref{eq:energy_identity} yields the central equivalence:
\begin{equation}
\label{eq:suboptimality_gap}
\mathcal{J}(\pi_{\mathrm{reas}}) - \mathcal{J}(\pi)
=
\beta \,
\mathbb{E}_{x \sim \mathcal{D}}
\left[
\mathrm{KL}\!\left(
\pi(\cdot \mid x)\, \| \, \pi_{\mathrm{reas}}(\cdot \mid x)
\right)
\right].
\end{equation}

We will then analysis the behaviour along trajectory, for simplicity we will consider the univariate Exponential Family (We will discuss why this is a suitable assumption in the end of this section).

We first summarize the basic assumptions we need.
\begin{itemize}
    \item \textbf{Optimization Path (Exponential Family):}
    \begin{equation}
        \pi_\lambda(y|x) = \frac{1}{Z_\lambda(x)} \pi_{\mathrm{inst}}(y|x) \exp\left( \lambda \cdot r(x,y) \right)
    \end{equation}
    where the \textbf{Energy Function} $Z_\lambda(x)$ is the normalization constant with $\{0,1\}$ reward $r$:
    \begin{equation}
        Z_\lambda(x) = \sum_{y \in \mathcal{Y}} \pi_{\mathrm{inst}}(y|x) \exp\left( \lambda \cdot r(x,y) \right)
    \end{equation}

    \item \textbf{Accuracy:}
    Accuracy is the expectation of the reward function under the current distribution:
    \begin{equation}
        R_\lambda(x) = \mathbb{E}_{y \sim \pi_\lambda(\cdot|x)} [r(x,y)] = \sum_{y \in \mathcal{Y}} \pi_\lambda(y|x) r(x,y)
    \end{equation}

    \item \textbf{Target Accuracy:}
    \begin{equation}
        R_{\mathrm{reas}}(x) = \mathbb{E}_{y \sim \pi_{\mathrm{reas}}(\cdot|x)} [r(x,y)]
    \end{equation}
\end{itemize}

Then we can prove the following theorem:
\begin{theorem}
\begin{equation}
    \mathrm{KL}(\pi_{\mathrm{reas}}(\cdot|x) \| \pi_\lambda(\cdot|x)) = D_{\mathrm{Bern}}(R_{\mathrm{reas}}(x) \| R_\lambda(x)),
\end{equation}
where $D_{\mathrm{Bern}}$ is the KL divergence between 2 Bernoulli distributions. 
\end{theorem}

\begin{proof}

    We need to calculate $\frac{\partial R_\lambda(x)}{\partial \lambda}$ to establish the link between parameter space and performance space.

    First, we calculate the derivative of $\log Z_\lambda(x)$ with respect to $\lambda$.
    \begin{align*}
        \frac{\partial}{\partial \lambda} \log Z_\lambda(x) &= \frac{1}{Z_\lambda(x)} \frac{\partial Z_\lambda(x)}{\partial \lambda} \nonumber \\
        &= \frac{1}{Z_\lambda(x)} \sum_{y \in \mathcal{Y}} \pi_{\mathrm{inst}}(y|x) \frac{\partial}{\partial \lambda} \exp\left(\lambda r(x,y)\right) \nonumber \\
        &= \frac{1}{Z_\lambda(x)} \sum_{y \in \mathcal{Y}} \pi_{\mathrm{inst}}(y|x) \exp\left(\lambda r(x,y)\right) \cdot r(x,y) \nonumber \\
        &= \sum_{y \in \mathcal{Y}} \underbrace{\frac{\pi_{\mathrm{inst}}(y|x) \exp\left(\lambda r(x,y)\right)}{Z_\lambda(x)}}_{\pi_\lambda(y|x)} r(x,y) \nonumber \\
        &= \mathbb{E}_{y \sim \pi_\lambda} [r(x,y)] \nonumber \\
        &= R_\lambda(x)
    \end{align*}

    Now we calculate the sensitivity of the probability density $\pi_\lambda(y|x)$ itself with respect to $\lambda$:
    \begin{align*}
        \frac{\partial}{\partial \lambda} \log \pi_\lambda(y|x) &= \frac{\partial}{\partial \lambda} \left( \log \pi_{\mathrm{inst}}(y|x) + \lambda r(x,y) - \log Z_\lambda(x) \right) \nonumber \\
        &= 0 + r(x,y) - \underbrace{\frac{\partial \log Z_\lambda(x)}{\partial \lambda}}_{R_\lambda(x)} \nonumber \\
        &= r(x,y) - R_\lambda(x)
    \end{align*}
    Using the logarithmic derivative identity $\partial \pi = \pi \cdot \partial \log \pi$:
    \begin{equation*}
        \frac{\partial \pi_\lambda(y|x)}{\partial \lambda} = \pi_\lambda(y|x) \left( r(x,y) - R_\lambda(x) \right)
    \end{equation*}

    \begin{align*}
        \frac{\partial R_\lambda(x)}{\partial \lambda} &= \frac{\partial}{\partial \lambda} \mathbb{E}_{\pi_\lambda}[r] = \sum_{y \in \mathcal{Y}} r(x,y) \frac{\partial \pi_\lambda(y|x)}{\partial \lambda} \nonumber \\
        &= \sum_{y \in \mathcal{Y}} r(x,y) \cdot \pi_\lambda(y|x) \left( r(x,y) - R_\lambda(x) \right) \nonumber \\
        &= \sum_{y \in \mathcal{Y}} \pi_\lambda(y|x) \left( r(x,y)^2 - r(x,y) R_\lambda(x) \right) \nonumber \\
        &= \mathbb{E}_{\pi_\lambda}[r^2] - R_\lambda(x) \underbrace{\mathbb{E}_{\pi_\lambda}[r]}_{R_\lambda(x)} \nonumber \\
        &= \text{Var}_{\pi_\lambda}(r) \quad \text{(Variance of the reward)}
    \end{align*}

    Since $r(x,y) \in \{0, 1\}$, we have $r^2 = r$.
    Therefore, the second moment equals the first moment: $\mathbb{E}[r^2] = \mathbb{E}[r] = R_\lambda(x)$.
    The variance formula simplifies to:
    \begin{equation*}
        \text{Var}(r) = R_\lambda(x) - (R_\lambda(x))^2 = R_\lambda(x)(1 - R_\lambda(x))
    \end{equation*}
    
Then we can find that
    \begin{equation*}
        d\lambda = \frac{1}{R_\lambda(x)(1 - R_\lambda(x))} dR_\lambda(x).
    \end{equation*}

    We need to calculate the rate of change of the distance $D(\lambda) = \mathrm{KL}(\pi_{\mathrm{reas}}(\cdot|x) \| \pi_\lambda(\cdot|x))$ between the current policy and the reasoning target.
    Note: $\pi_{\mathrm{reas}}$ is a fixed target distribution and does not change with $\lambda$.
    
    \begin{equation*}
        D(\lambda) = \sum_{y \in \mathcal{Y}} \pi_{\mathrm{reas}}(y|x) \log \frac{\pi_{\mathrm{reas}}(y|x)}{\pi_\lambda(y|x)} = \underbrace{\sum_{y} \pi_{\mathrm{reas}} \log \pi_{\mathrm{reas}}}_{\text{const}} - \sum_{y \in \mathcal{Y}} \pi_{\mathrm{reas}}(y|x) \log \pi_\lambda(y|x)
    \end{equation*}
    
    Differentiating with respect to $\lambda$:
    \begin{equation*}
        \frac{\partial D(\lambda)}{\partial \lambda} = - \sum_{y \in \mathcal{Y}} \pi_{\mathrm{reas}}(y|x) \frac{\partial}{\partial \lambda} \log \pi_\lambda(y|x)
    \end{equation*}
    
    We can know that $\frac{\partial \log \pi_\lambda}{\partial \lambda} = r(x,y) - R_\lambda(x)$:
    
    \begin{align*}
        \frac{\partial D(\lambda)}{\partial \lambda} &= - \sum_{y \in \mathcal{Y}} \pi_{\mathrm{reas}}(y|x) \left( r(x,y) - R_\lambda(x) \right) \nonumber \\
        &= - \left( \sum_{y \in \mathcal{Y}} \pi_{\mathrm{reas}}(y|x) r(x,y) - \sum_{y \in \mathcal{Y}} \pi_{\mathrm{reas}}(y|x) R_\lambda(x) \right)
    \end{align*}
    
    The first term is the expected reward under the target policy, i.e., $R_{\mathrm{reas}}(x)$.
    In the second term, $R_\lambda(x)$ is independent of the summation index $y$, so it can be factored out, and $\sum \pi_{\mathrm{reas}} = 1$.
    
    \begin{equation*}
        \frac{\partial D(\lambda)}{\partial \lambda} = - (R_{\mathrm{reas}}(x) - R_\lambda(x))
    \end{equation*}
    
Then we can find that
    \begin{equation*}
         d \mathrm{KL} = - (R_{\mathrm{reas}}(x) - R_\lambda(x)) d\lambda.
    \end{equation*}

    Now we eliminate the latent variable $\lambda$ and directly establish the physical relationship between KL divergence and accuracy $R$.

    \begin{equation}
        \frac{d D_{\mathrm{KL}}}{d R} = \frac{\partial D_{\mathrm{KL}} / \partial \lambda}{\partial R / \partial \lambda} = \frac{-(R_{\mathrm{reas}} - R)}{R(1-R)}
    \end{equation}
    (Subscripts $x$ and $\lambda$ are omitted for brevity, but dependencies remain).

    We integrate with respect to $R$.
    \begin{itemize}
        \item \textbf{Lower Limit:} Current accuracy $R$.
        \item \textbf{Upper Limit:} Target accuracy $R_{\mathrm{reas}}$.
        \item \textbf{Boundary Condition:} As $R \to R_{\mathrm{reas}}$, $\pi_\lambda \to \pi_{\mathrm{reas}}$, thus $\mathrm{KL} \to 0$.
    \end{itemize}
    
    Integration Equation:
    \begin{equation*}
        \int_{\mathrm{KL}}^{0} dK = \int_{R}^{R_{\mathrm{reas}}} - \frac{R_{\mathrm{reas}} - u}{u(1-u)} du
    \end{equation*}
    \begin{equation*}
        \mathrm{KL}(\pi_{\mathrm{reas}} \| \pi_\lambda) = \int_{R}^{R_{\mathrm{reas}}} \frac{R_{\mathrm{reas}} - u}{u(1-u)} du
    \end{equation*}

    Decomposing the integrand:
    \begin{equation*}
        \frac{R_{\mathrm{reas}} - u}{u(1-u)} = \frac{R_{\mathrm{reas}}}{u} - \frac{1-R_{\mathrm{reas}}}{1-u}
    \end{equation*}

    The primitive function is $F(u) = R_{\mathrm{reas}} \log u + (1-R_{\mathrm{reas}}) \log(1-u)$.
    Substituting the limits:
    \begin{align}
        \mathrm{KL} &= F(R_{\mathrm{reas}}) - F(R) \nonumber \\
        &= [R_{\mathrm{reas}} \log R_{\mathrm{reas}} + (1-R_{\mathrm{reas}}) \log(1-R_{\mathrm{reas}})] - [R_{\mathrm{reas}} \log R + (1-R_{\mathrm{reas}}) \log(1-R)] \nonumber \\
        &= R_{\mathrm{reas}} (\log R_{\mathrm{reas}} - \log R) + (1-R_{\mathrm{reas}}) (\log(1-R_{\mathrm{reas}}) - \log(1-R)) \nonumber \\
        &= R_{\mathrm{reas}}(x) \log \frac{R_{\mathrm{reas}}(x)}{R_\lambda(x)} + (1-R_{\mathrm{reas}}(x)) \log \frac{1-R_{\mathrm{reas}}(x)}{1-R_\lambda(x)}
    \end{align}
    
    The expression above strictly corresponds to the KL divergence between two Bernoulli distributions:
    \begin{equation*}
        \mathrm{KL}(\pi_{\mathrm{reas}}(\cdot|x) \| \pi_\lambda(\cdot|x)) = D_{\mathrm{Bern}}(R_{\mathrm{reas}}(x) \| R_\lambda(x)).
    \end{equation*}
    
\end{proof}

\subsection{A Gradient Flow View}
We will then show why the optimization trajectory of the problem can be treated as a univariate Exponential Family.

The objective is to maximize the expected reward under the data distribution $\mathcal{D}$ (with density $p(x)$), while applying entropy regularization relative to an instructor policy $\pi_{\mathrm{inst}}(y|x)$:
\begin{align*}
\mathcal{J}[\pi] &= \mathbb{E}_{x \sim p(x)} \left[ \mathbb{E}_{y \sim \pi(\cdot|x)} [r(x,y)] - \beta \, \mathrm{KL}\bigl(\pi(\cdot|x) \,\|\, \pi_{\mathrm{inst}}(\cdot|x)\bigr) \right] \\
&= \int p(x) \left( \sum_y \pi(y|x) \, r(x,y) - \beta \sum_y \pi(y|x) \ln \frac{\pi(y|x)}{\pi_{\mathrm{inst}}(y|x)} \right) dx.
\end{align*}

Due to the pointwise normalization constraint $\sum_y \pi(y|x) = 1$ for all $x$, we introduce a Lagrange multiplier function $\lambda(x)$:
\begin{equation*}
\mathcal{L}[\pi, \lambda] = \mathcal{J}[\pi] - \int p(x) \, \lambda(x) \left( \sum_y \pi(y|x) - 1 \right) dx.
\end{equation*}

Perturb $\pi \to \pi + \epsilon \eta$ (with $\sum_y \eta(y|x) = 0$). Then the Euclidean functional gradient is
\begin{equation*}
\frac{\delta \mathcal{L}}{\delta \pi(y|x)} = p(x) \left[ r(x,y) - \beta \log \frac{\pi(y|x)}{\pi_{\mathrm{inst}}(y|x)} - \beta - \lambda(x) \right].
\end{equation*}
The factor $p(x)$ implies that high-density samples dominate updates in Euclidean geometry, leading to poor handling of long-tail data.

For tangent vectors $\delta \pi_1, \delta \pi_2$ (satisfying $\sum_y \delta \pi_i(y|x) = 0$),
\begin{equation*}
\langle \delta \pi_1, \delta \pi_2 \rangle_{\mathcal{G}} = \int p(x) \sum_y \frac{\delta \pi_1(y|x) \, \delta \pi_2(y|x)}{\pi(y|x)} \, dx.
\end{equation*}
The inverse acts locally as
\begin{equation*}
(\mathcal{G}^{-1} g)(y|x) = \pi(y|x) \cdot \frac{g(y|x)}{p(x)}.
\end{equation*}

The natural gradient is $\tilde{\nabla} \mathcal{J} = \mathcal{G}^{-1} (\nabla_{\mathrm{Euc}} \mathcal{L})$:
\begin{equation*}
\tilde{\nabla} \mathcal{J}(y|x) = \pi(y|x) \left[ r(x,y) - \beta \log \frac{\pi(y|x)}{\pi_{\mathrm{inst}}(y|x)} - \beta - \lambda(x) \right].
\end{equation*}
The $p(x)$ terms cancel completely. Each $x$ evolves at a rate independent of its data density, achieving automatic decoupling.

The steepest ascent flow on the manifold is
\begin{equation*}
\partial_t \pi_t(y|x) = \tilde{\nabla} \mathcal{J}(y|x).
\end{equation*}
Substituting yields the replicator-like dynamics
\begin{equation*}
\partial_t \pi_t(y|x) = \pi_t(y|x) \left[ r(x,y) - \beta \log \frac{\pi_t(y|x)}{\pi_{\mathrm{inst}}(y|x)} - \beta - \lambda_t(x) \right],
\end{equation*}
where
\begin{equation*}
\lambda_t(x) = \mathbb{E}_{\pi_t(\cdot|x)} \left[ r(x,y) - \beta \log \frac{\pi_t(y|x)}{\pi_{\mathrm{inst}}(y|x)} - \beta \right].
\end{equation*}

Let $L_t(y|x) = \log \pi_t(y|x)$. Then
\begin{equation*}
\partial_t L_t(y|x) = r(x,y) - \beta (L_t(y|x) - \log \pi_{\mathrm{inst}}(y|x)) - \beta - \lambda_t(x).
\end{equation*}
This is a linear first-order ODE for each $x$. Using the integrating factor $e^{\beta t}$ and initial condition $\pi_0 = \pi_{\mathrm{inst}}$, the closed-form solution is
\begin{equation*}
\pi_t(y|x) = \frac{\pi_{\mathrm{inst}}(y|x) \exp\left( \frac{1 - e^{-\beta t}}{\beta} r(x,y) \right)}{Z_t(x)},
\end{equation*}
where the partition function is
\begin{equation*}
Z_t(x) = \sum_y \pi_{\mathrm{inst}}(y|x) \exp\left( \frac{1 - e^{-\beta t}}{\beta} r(x,y) \right).
\end{equation*}

\subsection{The Entropy Rule}

In this section, we will discuss the entropy related quantities in RL.

We can first define the gap function $\Delta(x)$ as:
\begin{equation*}
    \Delta_\lambda(x) \triangleq \text{CE}(\pi_{\mathrm{reas}}(\cdot|x), \pi_\lambda(\cdot|x)) - H(\pi_\lambda(\cdot|x))
\end{equation*}
Expanding the definitions, this is equivalent to the difference in expected log-likelihoods:
\begin{equation*}
    \Delta_\lambda(x) = \mathbb{E}_{y \sim \pi_\lambda(\cdot|x)} [\log \pi_\lambda(y|x)] - \mathbb{E}_{y \sim \pi_{\mathrm{reas}}(\cdot|x)} [\log \pi_\lambda(y|x)]
\end{equation*}

Let the perturbation parameter be $\delta = \frac{1}{\beta} - \lambda$. We assume $\delta$ is small ($t$ is not very small).

To compute the expectation under $\pi_{\mathrm{reas}}$ using samples from $\pi_\lambda$, we derive the likelihood ratio (density ratio) $w(y|x)$.

\begin{align*}
    w(y|x) &= \frac{\pi_{\mathrm{reas}}(y|x)}{\pi_\lambda(y|x)} \\
    &= \frac{Z_\lambda(x)}{Z^*(x)} \frac{\pi_{\mathrm{inst}}(y|x) \exp(r(x,y)/\beta)}{\pi_{\mathrm{inst}}(y|x) \exp(\lambda r(x,y))} \\
    &= \frac{Z_\lambda(x)}{Z^*(x)} \exp\left( \left(\frac{1}{\beta} - \lambda\right) r(x,y) \right) \\
    &= \frac{Z_\lambda(x)}{Z^*(x)} \exp\left( \delta \cdot r(x,y) \right)
\end{align*}

Using the normalization condition $\mathbb{E}_{y \sim \pi_{\mathrm{reas}}}[1] = 1$, we find the ratio of partition functions:
\begin{equation*}
    1 = \mathbb{E}_{y \sim \pi_\lambda(\cdot|x)} [w(y|x)] = \frac{Z_\lambda(x)}{Z^*(x)} \mathbb{E}_{y \sim \pi_\lambda(\cdot|x)} [\exp(\delta r(x,y))]
\end{equation*}
Thus, $\frac{Z^*(x)}{Z_\lambda(x)} = \mathbb{E}_{\pi_\lambda}[\exp(\delta r)]$, which corresponds to the Moment Generating Function.

We apply the Linear approximation by performing a first-order Taylor expansion around $\delta = 0$.

\begin{equation*}
    \exp(\delta r(x,y)) \approx 1 + \delta r(x,y)
\end{equation*}

\begin{equation*}
    \frac{Z^*(x)}{Z_\lambda(x)} = \mathbb{E}_{\pi_\lambda}[\exp(\delta r)] \approx \mathbb{E}_{\pi_\lambda}[1 + \delta r] = 1 + \delta \mathbb{E}_{\pi_\lambda}[r]
\end{equation*}
Using the approximation $(1+u)^{-1} \approx 1-u$:
\begin{equation*}
    \frac{Z_\lambda(x)}{Z^*(x)} \approx \frac{1}{1 + \delta \mathbb{E}_{\pi_\lambda}[r]} \approx 1 - \delta \mathbb{E}_{\pi_\lambda}[r]
\end{equation*}

Substituting back into $w(y|x)$ and ignoring $\mathcal{O}(\delta^2)$ terms:
\begin{align*}
    w(y|x) &\approx (1 - \delta \mathbb{E}_{\pi_\lambda}[r]) (1 + \delta r(x,y)) \\
    &\approx 1 + \delta r(x,y) - \delta \mathbb{E}_{\pi_\lambda}[r] \\
    &= 1 + \delta (r(x,y) - \mathbb{E}_{\pi_\lambda}[r])
\end{align*}

Now we estimate the difference in expectations for any statistic $A(y|x)$.
\begin{align*}
    \mathbb{E}_{\pi_{\mathrm{reas}}}[A] - \mathbb{E}_{\pi_\lambda}[A] &= \mathbb{E}_{\pi_\lambda}[A \cdot w] - \mathbb{E}_{\pi_\lambda}[A] \\
    &\approx \mathbb{E}_{\pi_\lambda} \left[ A \cdot (1 + \delta(r - \mathbb{E}_{\pi_\lambda}[r])) \right] - \mathbb{E}_{\pi_\lambda}[A] \\
    &= \mathbb{E}_{\pi_\lambda}[A] + \delta \left( \mathbb{E}_{\pi_\lambda}[A \cdot r] - \mathbb{E}_{\pi_\lambda}[A]\mathbb{E}_{\pi_\lambda}[r] \right) - \mathbb{E}_{\pi_\lambda}[A] \\
    &= \delta \cdot \mathrm{Cov}_{y \sim \pi_\lambda(\cdot|x)} (A(y|x), r(x,y))
\end{align*}

We substitute $A(y|x) = \log \pi_\lambda(y|x)$ into the formula. Note that our gap definition is $\Delta = \mathbb{E}_{\pi_\lambda} - \mathbb{E}_{\pi_{\mathrm{reas}}}$, so we introduce a negative sign:
\begin{equation*}
    \Delta_\lambda(x) \approx - \delta \cdot \mathrm{Cov}_{y \sim \pi_\lambda(\cdot|x)} (\log \pi_\lambda(y|x), r(x,y))
\end{equation*}

Therefore, the final approximation is:

\begin{equation*}
   \text{CE}(\pi_{\mathrm{reas}}, \pi_\lambda) - H(\pi_\lambda) \approx - \left( \frac{1}{\beta} - \lambda \right) \cdot \mathrm{Cov}_{y \sim \pi_\lambda(\cdot|x)} (\log \pi_\lambda(y|x), r(x,y)).
\end{equation*}

From Jensen's inequality, we know that
$$
\mathbb{E}_x D_{\mathrm{Bern}}(R_{\mathrm{reas}}(x) \| R_\lambda(x)) \ge D_{\mathrm{Bern}}{ (\mathbb{E}_x R_{\mathrm{reas}}(x) \| \mathbb{E}_x R_\lambda(x) )}.
$$

We can finally get the following corollary, which is emperically observed:
\begin{corollary}
When $\mathbb{E}_x R_{\mathrm{reas}}(x) =1$ and the learning rate $\alpha_n = \tfrac{\exp^{-\beta n}}{\beta}$, then there exists constants $a,b$ s.t. $\mathbb{E}_x R_{n}(x) \approx b -a \exp( \mathbb{E}_x H(\pi_{n+1}(\cdot|x)))$.    
\end{corollary}
\begin{proof}
    Denote $R^* = \mathbb{E}_x R_{\mathrm{reas}}(x)$ and $R = \mathbb{E}_x R_{\lambda}(x)$.

    Observe that
    \[
    \exp\!\bigl(D_{\mathrm{Bern}}(R_{\mathrm{reas}}(x) \| R_\lambda(x))\bigr)
    = \left(\frac{R^*}{R}\right)^{R^*} \left(\frac{1-R^*}{1-R}\right)^{1-R^*}
    = \left(1 + \frac{R^*-R}{R}\right)^{R^*} \left(1 - \frac{R^*-R}{1-R}\right)^{1-R^*}.
    \]
    
    Applying a first-order Taylor expansion yields the approximation
    \[
    \exp\!\bigl(D_{\mathrm{Bern}}(R_{\mathrm{reas}}(x) \| R_\lambda(x))\bigr)
    \approx 1 + \left( \frac{\log R^*}{R} - \frac{\log (1-R^*)}{1-R} \right) (R^* - R).
    \]
    
    Taking expectations over $x$, we obtain
    \begin{align*}
    &\mathbb{E}_x \exp\!\bigl(\mathrm{KL}(\pi_{\mathrm{reas}}(\cdot|x) \| \pi_\lambda(\cdot|x))\bigr) \\
    &\quad= \mathbb{E}_x \exp\!\bigl(D_{\mathrm{Bern}}(R_{\mathrm{reas}}(x) \| R_\lambda(x))\bigr) \\
    &\quad\approx 1 + \mathbb{E}_x \left[ \left( \frac{\log R^*}{R} - \frac{\log (1-R^*)}{1-R} \right) (R^* - R) \right].
    \end{align*}
    
    For small gaps (where the second-order terms are negligible), this further simplifies to
    \[
    \mathbb{E}_x \exp\!\bigl(\mathrm{KL}(\pi_{\mathrm{reas}}(\cdot|x) \| \pi_\lambda(\cdot|x))\bigr)
    \approx 1 + C \, \mathbb{E}_x (R^* - R)
    = b - \mathbb{E}_x R_n(x),
    \]
    for some constants $C$ and $b$.
    
    From Theorem~2 in \cite{cui2025entropy}, we have $\text{CE}(\pi_{\mathrm{reas}}, \pi_{n+1}) - H(\pi_n) \approx H(\pi_{n+1}) - H(\pi_n)$.  
    Noting that $\mathrm{KL}(\pi_{\mathrm{reas}}, \pi_{n+1}) = \text{CE}(\pi_{\mathrm{reas}}, \pi_{n+1}) - H(\pi_{\mathrm{reas}})$, the desired approximation follows.

\end{proof}

 \section{Conclusion}

In this paper, we presented a unified and principled theoretical framework for understanding KL-regularized reinforcement learning in large language models through the lens of Energy-Based Models (EBMs). By making the energy structure induced by the KL-regularized objective explicit, we showed that a broad class of alignment and reasoning algorithms—including instruction tuning and RLVR—admit a common variational interpretation as conditional EBMs. 

For instruction-tuned models, we established that under mild and interpretable assumptions, the induced transition kernel satisfies a multiplicative detailed-balance condition with respect to an explicit potential function. This structural result yields immediate theoretical consequences, including monotonicity of the KL divergence to a canonical stationary distribution, quantitative bounds on hitting times to low-potential (high-quality) states, and exponential convergence rates governed by the absolute spectral gap of the associated graph Laplacian. These results provide a rigorous explanation for empirically observed behaviors.

Extending the analysis to RLVR, we derived an exact equivalence between the RL objective and the minimization of expected KL divergence to an optimal reasoning distribution. By analyzing the optimization trajectory as a univariate exponential family induced by natural gradient flow, we obtained closed-form identities that tightly connect optimization geometry and performance metrics. In particular, for binary rewards, the KL gap reduces exactly to a Bernoulli KL divergence between current and target accuracies, yielding a precise theoretical account of the empirically observed entropy--accuracy trade-off.

Beyond explaining existing empirical phenomena, our results suggest a broader message: many seemingly complex behaviors of RL-trained LLMs are consequences of the underlying energy-based structure imposed by KL regularization.

 \clearpage

\bibliography{ref}

@article{openai2023gpt,
  title={Gpt-4 technical report. arxiv 2303.08774},
  author={OpenAI, R},
  journal={View in Article},
  volume={2},
  number={5},
  year={2023}
}

@article{touvron2023llama,
  title={Llama: Open and efficient foundation language models (2023)},
  author={Touvron, Hugo and Lavril, Thibaut and Izacard, Gautier and Martinet, Xavier and Lachaux, Marie-Anne and Lacroix, Timoth{\'e}e and Rozi{\`e}re, Baptiste and Goyal, Naman and Hambro, Eric and Azhar, Faisal and others},
  journal={arXiv preprint arXiv:2302.13971},
  year={2023}
}

@article{schulman2017proximal,
  title={Proximal Policy Optimization Algorithms},
  author={Schulman, John and Wolski, Filip and Dhariwal, Prafulla and Radford, Alec and Klimov, Oleg},
  journal={arXiv preprint arXiv:1707.06347},
  year={2017}
}

@article{rafailov2023direct,
  title={Direct Preference Optimization: Your Language Model is Secretly a Reward Model},
  author={Rafailov, Rafael and Sharma, Archit and Mitchell, Eric and Manning, Christopher D and Ermon, Stefano and Finn, Chelsea},
  journal={arXiv preprint arXiv:2305.18290},
  year={2023}
}

@article{shao2024deepseekmath,
  title={DeepSeekMath: Pushing the Limits of Mathematical Reasoning in Open Language Models},
  author={Shao, Zhihong and others},
  journal={arXiv preprint arXiv:2402.03300},
  year={2024}
}

@incollection{LeCun2006,
author = {LeCun, Yann and Chopra, Sumit and Hadsell, Raia and Ranzato, Marc'Aurelio and Huang, Fu-Jie},
title = {A Tutorial on Energy-Based Learning},
booktitle = {Predicting Structured Data},
editor = {Bakir, Gert and Hofmann, Thomas and Schölkopf, Bernhard and Smola, Alexander and Vishwanathan, S. V. N.},
publisher = {MIT Press},
year = {2006},
pages = {Introduction--Review},
note = {Available online: \url{https://yann.lecun.com/exdb/publis/pdf/lecun-06.pdf}}

}

@article{Bakhtin2021,
author = {Anton Bakhtin and Yuntian Deng and Sam Gross and Myle Ott and Marc'Aurelio Ranzato and Arthur Szlam},
title = {Residual Energy-Based Models for Text},
journal = {Journal of Machine Learning Research},
volume = {22},
number = {40},
pages = {1--41},
year = {2021},
url = {http://jmlr.org/papers/volume22/20-326/20-326.pdf}

}

@article{Du2019,
author = {Yilun Du and Igor Mordatch},
title = {Implicit Generation and Generalization in Energy-Based Models},
journal = {arXiv preprint arXiv:1903.08689},
year = {2019},
url = {https://arxiv.org/abs/1903.08689}

}

@article{Song2021,
author = {Yang Song and Diederik P. Kingma},
title = {How to Train Your Energy-Based Models},
journal = {arXiv preprint arXiv:2101.03288},
year = {2021},
url = {https://arxiv.org/abs/2101.03288}

}

@article{Xu2024_edlm,
author = {Minkai Xu and Tomas Geffner and Karsten Kreis and Weili Nie and Yilun Xu and Jure Leskovec and Stefano Ermon and Arash Vahdat},
title = {Energy-Based Diffusion Language Models for Text Generation},
journal = {arXiv preprint arXiv:2410.21357},
year = {2024},
note = {ICLR 2025; PDF: \url{https://arxiv.org/abs/2410.21357}}

}

@inproceedings{Haarnoja2017,
author = {Tuomas Haarnoja and Haoran Tang and Pieter Abbeel and Sergey Levine},
title = {Reinforcement Learning with Deep Energy-Based Policies},
booktitle = {Proceedings of the 34th International Conference on Machine Learning (ICML)},
series = {Proceedings of Machine Learning Research},
volume = {70},
pages = {1352--1361},
year = {2017},
url = {http://proceedings.mlr.press/v70/haarnoja17a/haarnoja17a.pdf}

}

@article{Du2019_planning,
author = {Yilun Du and Toru Lin and Igor Mordatch},
title = {Model Based Planning with Energy Based Models},
journal = {arXiv preprint arXiv:1909.06878},
year = {2019},
url = {https://arxiv.org/abs/1909.06878}

}

@article{Chao2024_EBFlow,
author = {Chen-Hao Chao and Chien Feng and Wei-Fang Sun and Cheng-Kuang Lee and Simon See and Chun-Yi Lee},
title = {Maximum Entropy Reinforcement Learning via Energy-Based Normalizing Flow},
journal = {arXiv preprint arXiv:2405.13629},
year = {2024},
note = {NeurIPS 2024 (Proceedings).},
url = {https://arxiv.org/abs/2405.13629}

}

@article{Messaoud2024_S2AC,
author = {Safa Messaoud and Billel Mokeddem and Zhenghai Xue and Linsey Pang and Bo An and Haipeng Chen and Sanjay Chawla},
title = {S$^2$AC: Energy-Based Reinforcement Learning with Stein Soft Actor Critic},
journal = {arXiv preprint arXiv:2405.00987},
year = {2024},
url = {https://arxiv.org/abs/2405.00987}

}

@article{song2025detailed,
  title={Detailed balance in large language model-driven agents},
  author={Song, Zhuo-Yang and Cao, Qing-Hong and Luo, Ming-xing and Zhu, Hua Xing},
  journal={arXiv preprint arXiv:2512.10047},
  year={2025}
}

@article{cui2025entropy,
  title={The entropy mechanism of reinforcement learning for reasoning language models},
  author={Cui, Ganqu and Zhang, Yuchen and Chen, Jiacheng and Yuan, Lifan and Wang, Zhi and Zuo, Yuxin and Li, Haozhan and Fan, Yuchen and Chen, Huayu and Chen, Weize and others},
  journal={arXiv preprint arXiv:2505.22617},
  year={2025}
}

@article{kumar2024training,
  title={Training language models to self-correct via reinforcement learning},
  author={Kumar, Aviral and Zhuang, Vincent and Agarwal, Rishabh and Su, Yi and Co-Reyes, John D and Singh, Avi and Baumli, Kate and Iqbal, Shariq and Bishop, Colton and Roelofs, Rebecca and others},
  journal={arXiv preprint arXiv:2409.12917},
  year={2024}
}

@inproceedings{wang2024math,
  title={Math-shepherd: Verify and reinforce llms step-by-step without human annotations},
  author={Wang, Peiyi and Li, Lei and Shao, Zhihong and Xu, Runxin and Dai, Damai and Li, Yifei and Chen, Deli and Wu, Yu and Sui, Zhifang},
  booktitle={Proceedings of the 62nd Annual Meeting of the Association for Computational Linguistics (Volume 1: Long Papers)},
  pages={9426--9439},
  year={2024}
}
\bibliographystyle{iclr2026_conference}


\end{document}